\documentclass[12pt]{article}
\usepackage[utf8]{inputenc}
\usepackage{amsmath}
\usepackage{mathtools}
\usepackage[letterpaper,top=1in,bottom=1in,left=1in,right=1in]{geometry}
\usepackage{bm}
\usepackage{hyperref}
\usepackage{amsthm}
\renewcommand\a{\bm{a}}
\newcommand\bmu{\bm{\mu}}
\newcommand\g{\bm{g}}
\newcommand\w{\bm{w}}
\newcommand\x{\bm{x}}
\newcommand\y{\bm{y}}
\newcommand\z{\bm{z}}

\newcommand\bz{\bm{0}}
\newcommand\R{\bm{\mathrm{R}}}
\newcommand{\norm}[1] {\left \| #1 \right \|} %norm
\newtheorem{theorem}{Theorem}
\newtheorem{corollary}{Corollary}[theorem]

\title{Recovery of a mixture of Gaussians by sum-of-norms clustering}

\author{Tao Jiang\thanks{Department of Combinatorics \&  Optimization, University of Waterloo, Waterloo, Ontario,  Canada N2L 3G1, {\tt tao.jiang@uwaterloo.ca}.} \and  
Stephen Vavasis\thanks{Department of Combinatorics \&  Optimization, University of Waterloo, Waterloo, Ontario,  Canada N2L 3G1, {\tt vavasis@uwaterloo.ca}.  Research supported in part by a Discovery Grant from the Natural Science and Engineering Research Council (NSERC) of Canada.} \and Chen Wen Zhai\thanks{Department of Combinatorics \& Optimization and Department of Statistics \& Actuarial  Science,
 University of Waterloo, Waterloo, Ontario, Canada, N2L 3G1,  {\tt sabrina.zhai@edu.uwaterloo.ca}. }}
\date{}
\begin{document}

\maketitle

%\pagebreak
\begin{abstract}
Sum-of-norms clustering is a method for assigning  $n$ points in $\R^d$ to $K$ clusters,  $1\le K\le n$, using convex optimization.  Recently, Panahi et al.\ proved that sum-of-norms clustering is guaranteed to recover a mixture of Gaussians under the restriction that the number of samples is not too large.  The purpose of this note is to lift this restriction, i.e., show that sum-of-norms clustering with equal weights can recover a mixture of Gaussians even as the number of samples tends to infinity.  Our proof relies on an interesting characterization of clusters computed by sum-of-norms clustering that was developed inside a proof of the agglomeration conjecture by Chiquet et al.  Because we believe this theorem has independent interest, we restate and reprove the Chiquet et al.\ result herein.
%\noindent keywords:
\end{abstract}

%\newpage

\section{Introduction}
Clustering is perhaps the most central problem in unsupervised machine learning and has been studied for over 60 years \cite{shais}.  The problem may be stated informally as follows.  One is given $n$ points, $\a_1,\ldots,\a_n$ lying in $\R^d$.  One seeks to partition $\{1,\ldots,n\}$ into $K$ sets $C_1,\ldots,C_K$ such that the $\a_i$'s for $i\in C_m$ are closer to each other than to the $\a_i$'s for $i\in C_{m'}$, $m'\ne m$.  

Clustering is usually posed as a nonconvex optimization problem, and therefore prone to nonoptimal local minimizers, but 
Pelckmans et al.\ \cite{pelckmans}, Hocking et al.\ \cite{hocking}, and Lindsten et al.\ \cite{lindsten} proposed the following convex formulation for the clustering problem:
\begin{equation}
    \min_{\x_1,\ldots,\x_n\in\R^d} \frac{1}{2}\sum_{i=1}^n \norm{\x_i-\a_i}^2 +\lambda\sum_{i<j}\norm{\x_i-\x_j}.
    \label{eq:son-clustering}
\end{equation}
This formulation is known in the literature as sum-of-norms clustering, convex clustering, or clusterpath clustering.
  Let $\x_1^*,\ldots,\x_n^*$ be the optimizer.  (Note: \eqref{eq:son-clustering} is strongly convex, hence the optimizer exists and is unique.) 
  The assignment to clusters is given by the $\x_i^*$'s: for $i,i'$, if $\x_i^*=\x_{i'}^*$ then $i,i'$ are assigned to the same cluster, else they are assigned to different clusters.  It is apparent that for $\lambda=0$, each $\a_i$ is assigned to a different cluster (unless $\a_i=\a_{i'}$ exactly), whereas for $\lambda$ sufficiently large, the second summation drives all the $\x_i$'s to be equal (and hence there is one big cluster).  Thus, the parameter $\lambda$ controls the number of clusters produced by the formulation. 

Throughout this paper, we assume that all norms are Euclidean, although \eqref{eq:son-clustering} has also been considered for other norms.  In addition,  some authors insert nonnegative weights in front of the the terms in the above summations.  Our results, however, require all weights identically 1.

Panahi et al.\ \cite{Panahi} developed several recovery theorems as well as a first-order optimization method for solving \eqref{eq:son-clustering}.  Other authors, e.g., Sun et al.\ \cite{dsun1} have since extended these results.  One of Panahi et al.'s results pertains to a mixture of Gaussians, which is the following generative model for producing the data $\a_1,\ldots,\a_n$.  The parameters of the model are $K$ means $\bmu_1,\ldots,\bmu_K\in\R^d$, $K$ variances $\sigma_1^2,\ldots,\sigma_k^2$, and $K$ probabilities $w_1,\ldots,w_k$, all positive and summing to 1.  One draws $n$ i.i.d.\ samples as follows.  First, an index $m\in\{1,\ldots,K\}$ is selected at random according to probabilities $w_1,\ldots,w_K$.  Next, a point $\a$ is chosen according to the spherical Gaussian distribution $N(\bmu_m,\sigma_m^2I)$.

Panahi et al.\ proved that for the appropriate choice of $\lambda$,
sum-of-norms clustering formulation \eqref{eq:son-clustering} will exactly recover a mixture of Gaussians (i.e., each point will be labeled with $m$ if it was selected from $N(\bmu_m,\sigma_m^2I)$) provided that for all $m,m'$, $1\le m<m'\le K$,
\begin{equation}
\norm{\bmu_m-\bmu_{m'}}\ge \frac{CK\sigma_{\max}}{w_{\min}}\mathrm{polylog}(n).
\label{eq:panahibound}
\end{equation}
One issue with this bound is that as the number of samples $n$ tends to infinity, the bound seems to indicate that distinguishing the clusters becomes increasingly difficult (i.e., the $\bmu_m$'s have to be more distantly separated as $n\rightarrow\infty$).  

The reason for this aspect of their bound is that their proof technique requires a gap of positive width (i.e., a region of $\R^d$ containing no sample points) between $\{\a_i:i\in C_m\}$ and $\{\a_i:i\in C_{m'}\}$ whenever $m\ne m'$.  Clearly, such a gap cannot exist in the mixture-of-Gaussians distribution as the number of samples tends to infinity.

The purpose of this note is to prove that \eqref{eq:son-clustering} can recover a mixture of Gaussians even as $n\rightarrow\infty$.  This is the content of Theorem~\ref{thm:mainrecovery} in Section~\ref{sec:mainrecovery} below.  Naturally, under this hypothesis we cannot hope to correctly label all samples since, as $n\rightarrow\infty$, some of the samples associated with one mean will be placed arbitrarily close to another mean.  Therefore, we are content in showing that \eqref{eq:son-clustering} can correctly cluster the points lying within some fixed number of standard-deviations for each mean.  Radchenko and Mukherjee \cite{Radchenko} have previously analyzed the special case of mixture of Gaussians with $K=2$, $d=1$ under slightly different hypotheses.

Our proof technique requires a cluster characterization theorem for sum-of-norms clustering derived by Chiquet et al.\ \cite{chiquet}.  This theorem is not stated by these authors as a theorem, but instead appears as a sequence of steps inside a larger proof in a ``supplementary material'' appendix to their paper.  Because we believe that this theorem is of independent interest, we restate it below and for the sake of completeness provide the proof (which is the same as the proof appearing in Chiquet et al.'s supplementary material).  This material appears in Section~\ref{sec:chiquetproof}.

\section{Cluster characterization theorem}
\label{sec:chiquetproof}

The following theorem is due to Chiquet et al.\ \cite{chiquet} but is not stated as a theorem by these authors; instead it appears as a sequence of steps in a proof of the agglomeration conjecture.  Refer to the next section for a discussion of the agglomeration conjecture.  We restate the theorem here because it is needed for our analysis and because we believe it is of independent interest.

\begin{theorem}
Let $\x_1^*,\ldots,\x_n^*$ denote the optimizer of \eqref{eq:son-clustering}.  For notational ease, let  $\x^*$ denote the concatenation of these vectors into a single $nd$-vector.
Suppose that $C$ is a nonempty subset of $\{1,\ldots,n\}$.

(a) Necessary condition: If for some $\hat\x \in \R^d$, $\x_i^*=\hat\x$ for $i\in C$ and $\x_i^*\ne\hat\x$ for $i\notin C$ (i.e., $C$ is exactly one cluster determined by \eqref{eq:son-clustering}), then there exist $\z_{ij}^*$ for $i,j\in C$, $i\ne j$, which solve
\begin{equation}
\begin{aligned}
\a_i-\frac{1}{|C|}\sum_{l\in C}\a_l&=
\lambda\sum_{j\in C-\{i\}} \z_{ij}^* &&\forall i\in C,\\
\norm{\z_{ij}^*} &\leq 1 &&\forall i,j\in C, i\ne j, \\ 
\z_{ij}^* &= -\z_{ji}^*&& \forall i,j\in C,  i\ne j.
\end{aligned}
\label{eq:zstar1}
\end{equation}

(b) Sufficient condition: Suppose there exists a solution $\z_{ij}^* $ for $j\in C-\{i\}$, $i\in C$ to the conditions \eqref{eq:zstar1}. Then there exists an $\hat\x \in \R^d$ such that the minimizer $\x^*$ of \eqref{eq:son-clustering} satisfies $\x_i^*=\hat\x$ for $i\in C$.
\label{thm:clustchar}
\end{theorem}

\noindent
{\bf Note}: This theorem is an almost exact characterization of clusters that are determined by formulation \eqref{eq:son-clustering}.  The only gap between the necessary and sufficient conditions is that the necessary condition requires that $C$ be exactly all the points in a cluster, whereas the sufficient condition is sufficient for $C$ to be a subset of the points in a cluster.  The sufficient condition is notable because it does not require any hypothesis about the other $n-|C|$ points occurring in the input.

\begin{proof} (Chiquet et al.)
\underline{Proof for Necessity (a)} \\
As $\x^*$ is the minimizer of the problem (\ref{eq:son-clustering}), and this objective
function, call it $f(\x)$, is convex, it follows
that $\bz\in\partial f(\x^*)$, where $\partial f(\x^*)$ denotes the subdifferential, that is, the set of subgradients of 
$f$ at $\x^*$. (See, e.g., \cite{hiriarturrutylemarechal} for background on convex analysis).  Written explicitly in terms of the derivative of the squared-norm and subdifferential of the norm, this means that $\x^*$ satisfies the following condition:
\begin{equation}
    \x_i^*-\a_i+\lambda\sum_{j\ne i}\w_{ij}^*=\bz \qquad\forall i=1,\ldots,n,
    \label{eq:KKTcond}
\end{equation}
where $\w_{ij}^*$, $i=1,\ldots,n$, $j=1,\ldots,n$, $i\ne j$, are subgradients of the Euclidean norm function
satisfying
$$\w_{ij}^* = \left\{
\begin{array}{ll}
  \frac{\x_i^*-\x_j^*}{\norm{\x_i^*-\x_j^*}}, & \mbox{for $\x_i^*\ne \x_j^*$}, \\
  \mbox{arbitrary point in $B(\bz,1)$}, & \mbox{for $\x_i^*=\x_j^*$},
\end{array}
\right.
$$
with the requirement that $\w_{ij}^*=-\w_{ji}^*$ in the second case.
Here, $B(\x,r)$ is notation for the closed Euclidean ball centered at $\x$ of radius $r$.
Since $\x_i^*=\hat\x$ for $i\in C$, $\x_i^*\ne\hat\x$ for $i\notin C$, the KKT condition for $i\in C$ is rewritten as 
\begin{equation}
\hat\x-\a_i+\lambda\sum_{j\notin C}\frac{\hat \x-\x_j^*}{\norm{\hat\x-\x_j^*}}
+\lambda\sum_{j\in C-\{i\}} \w_{ij}^*=\bz,
\label{eq:hatx1}
\end{equation}

Define $\z_{ij}^*=\w_{ij}^*$ for $i,j\in C$, $i\ne j$.  Then
$$\norm{\z_{ij}^*} \leq 1, \z_{ij}^* = -\z_{ji}^*, \forall i,j\in C, i\ne j.$$
Substitute $\w_{ij}^* = \z_{ij}^*$ into the equation (\ref{eq:hatx1}) to obtain
\begin{equation}
\hat\x-\a_i+\lambda\sum_{j\notin C}\frac{\hat \x-\x_j^*}{\norm{\hat\x-\x_j^*}}
+\lambda\sum_{j\in C-\{i\}} \z_{ij}^*=\bz,
\label{eq:hatx4}
\end{equation}
Sum the preceding equation over $i\in C$, noticing that the last term
cancels out, leaving
$$
  |C|\hat\x-\sum_{i\in C}\a_i+\lambda|C|\sum_{j\notin C}\frac{\hat \x-\x_j^*}{\norm{\hat\x-\x_j^*}}=\bz,
$$
which is rearranged to (renaming $i$ to $l$):
\begin{equation}
  \lambda\sum_{j\notin C}\frac{\hat \x-\x_j^*}{\norm{\hat\x-\x_j^*}}=
  -\hat\x+\frac{1}{|C|}\sum_{l\in C}\a_l.
  \label{eq:hatx2}
\end{equation}
Subtract \eqref{eq:hatx2} from \eqref{eq:hatx4}, simplify and 
rearrange to obtain
\begin{equation}
\a_i-\frac{1}{|C|}\sum_{l\in C}\a_l=
\lambda\sum_{j\in C-\{i\}} \z_{ij}^* \qquad\forall i\in C,
\label{eq:zstar2}
\end{equation}
as desired.

\underline{Proof for Sufficiency (b)}\\
We will show that at the solution of \eqref{eq:son-clustering},
all the $\x_i^*$'s for $i\in C$ have a common value under the hypothesis that $\z_{ij}^* $ is a solution to the equation \eqref{eq:zstar1} for $i,j\in C$, $i\ne j$.

First, define the following intermediate 
problem.  Let $\tilde \a$ denote the centroid of $\a_i$ for $i\in C$:
$$\tilde \a = \frac{1}{|C|}\sum_{l\in C}\a_l.$$
Consider the weighted problem sum-of-norms clustering
problem with unknowns as follows: one unknown $\x\in\R^d$ is associated
with $C$, and one unknown $\x_j$ is associated with each $j\notin C$
(for a total of with $n-|C|+1$ unknown vectors):
\begin{equation}
  \min_{\x;\x_j}
  \frac{|C|}{2}\cdot \norm{\x-\tilde\a}^2 + \frac{1}{2}\sum_{j\notin C}\norm{\x_j-\a_j}^2
  +\lambda|C|\sum_{j\notin C} \norm{\x-\x_j} +
  \lambda\sum_{\substack {{i,j\notin C} \\{i<j}}}\norm{\x_i-\x_j}.
  \label{eq:sonclust2}
\end{equation}
This problem, being strongly convex, has a unique optimizer; denote
the optimizing vectors $\tilde \x$ and $\tilde \x_j$ for $j\notin C$.

First, let us
consider the optimality conditions for \eqref{eq:sonclust2}, which are:
\begin{align}
  |C|(\tilde\x-\tilde\a)+\lambda|C|\sum_{j\notin C}\g_j&=\bz,
  \label{eq:sonclust2_opt1} \\
  \tilde\x_i-\a_i-\lambda|C|\g_i+\lambda\sum_{j\notin C\cup\{i\}}\y_{ij}&=\bz
  \qquad \forall i\notin C,
  \label{eq:sonclust2_opt2}
\end{align}
with subgradients defined as follows:
$$\g_j=\left\{
\begin{array}{ll}
  \frac{\tilde\x-\tilde\x_j}{\norm{\tilde\x-\tilde\x_j}}, & \mbox{for $\tilde\x_j\ne\tilde\x$}, \\
    \mbox{arbitrary in $B(\bz,1)$}, & \mbox{for $\tilde\x_j=\tilde\x$},
\end{array}
\right.
\qquad\forall j\notin C,$$
and
$$\y_{ij} = \left\{
\begin{array}{ll}
  \frac{\tilde\x_i-\tilde\x_j}{\norm{\tilde\x_i-\tilde\x_j}}, &
  \mbox{for $\tilde\x_i\ne\tilde\x_j$}, \\
  \mbox{arbitrary in $B(\bz,1)$}, &
  \mbox{for $\tilde\x_i=\tilde\x_j$},
\end{array}
\right.
\qquad\forall i,j\notin C, i\ne j,$$
with the proviso that in the second case, $\y_{ij}=-\y_{ji}$.

We claim that the solution for \eqref{eq:son-clustering}
given by defining $\x_i^*=\tilde\x$ for
$i\in C$ while keeping the $\x_j^*=\tilde\x_j$ for $j\notin C$, where
$\tilde\x$ and $\tilde\x_j$ are the optimizers for \eqref{eq:sonclust2}
as in the last few paragraphs, 
is optimal for \eqref{eq:son-clustering}, which proves the main result.
To show that this solution is optimal for \eqref{eq:son-clustering},
we need to provide subgradients
to establish the necessary condition. Define
$\w_{ij}$ to be the subgradients of $\x_i\mapsto \norm{\x_i- \tilde\x_j^*}$ evaluated
at $\tilde\x_i^*$ as follows:
\begin{align*}
  \w_{ij}&=\g_{j}  && \mbox{for $i\in C, j\notin C$}, \\
  \w_{ij}&=\y_{ij} &&\mbox{for $i,j\notin C$, $i\ne j$}, \\
  \w_{ij}&=\z_{ij}^*&&\mbox{for $i,j\in C$, $i\ne j$,}
\end{align*}
Before confirming that the necessary condition is satisfied, we first
need to confirm that these are all valid subgradients.  In the
case that $i\in C$, $j\notin C$, we have constructed $\g_{j}$
to be a valid subgradient of $\x\mapsto \norm{\x-\tilde \x_j}$ evaluated
at $\tilde \x$, and we have taken $\x_i^*=\tilde\x$, $\x_j^*=\tilde\x_j$.

In the case that $i,j\notin C$, we have construct $\y_{ij}$
to be a valid subgradient of $\x\mapsto \norm{\x-\tilde\x_j}$
evaluated at $\tilde\x_i$, and we have taken $\x_i^*=\tilde\x_i$, $\x_j^*=\tilde\x_j$. 

In the case that $i,j\in C$, by construction $\x_i^*=\x_j^*=\tilde\x$,
so any vector in $B(\bz,1)$ is a valid subgradient of $\x\mapsto\norm{\x-\tilde\x_j}$
evaluated $\tilde\x_i$. Note that since
$\z_{ij}^*\in B(\bz,1)$, then $\w_{ij}$ defined above also lies in $B(\bz,1)$.

Now we check the necessary conditions for optimality in \eqref{eq:son-clustering}.  First,
consider an $i\in C$:
\begin{align*}
  \tilde \x_i^*-\a_i+\lambda\sum_{j\ne i}\w_{ij}
  &=
  \tilde \x - \a_i+\lambda \sum_{j\in C-\{i\}}\w_{ij}+ \lambda\sum_{j\notin C}\w_{ij} \\
  &=
  \tilde \x - \a_i+\lambda \sum_{j\in C-\{i\}}\z_{ij}^*+\lambda\sum_{j\notin C}\g_j
  \\
  &=
  \tilde \x - \a_i+\a_i-\frac{1}{|C|}\sum_{l\in C}\a_l
  +\lambda\sum_{j\notin C}\g_j
  &&\mbox{(by \eqref{eq:zstar1})}
  \\
  &=
  \tilde\x -\tilde\a  +\lambda\sum_{j\notin C}\g_j
  \\
  & = \bz
  &&\mbox{(by \eqref{eq:sonclust2_opt1}).}
\end{align*}
Then we check for $i\notin C$:
\begin{align*}
  \tilde \x_i^*-\a_i+\lambda\sum_{j\ne i}\w_{ij}
  &=
  \tilde \x_i-\a_i+\lambda\sum_{j\in C}\w_{ij}+\lambda \sum_{j\notin C\cup\{i\}}
  \w_{ij} \\
  &= 
  \tilde \x_i-\a_i+\lambda\sum_{j\in C}(-\g_i)+\lambda \sum_{j\notin C\cup\{i\}}
  \y_{ij} \\
  & =
  \tilde \x_i-\a_i-\lambda|C|\g_i+\lambda \sum_{j\notin C\cup\{i\}}
  \y_{ij} \\
  &= \bz && \mbox{(by \eqref{eq:sonclust2_opt2}).}
\end{align*}
\end{proof}

\section{Agglomeration Conjecture}
Recall that when $\lambda=0$, each $\a_i$ is in its own cluster in the solution to \eqref{eq:son-clustering} (provided the $\a_i$'s are distinct), whereas for sufficiently large $\lambda$, all the points are in one cluster.
Hocking et al.\ \cite{hocking} conjectured that sum-of-norms clustering with equal weights has the following agglomeration property: as $\lambda$ increases, clusters merge with each other but never break up.  This means that the solutions to \eqref{eq:son-clustering} as $\lambda$ ranges over $[0,\infty)$ induce a tree of hierarchical clusters on the data.

This conjecture was proved by Chiquet et al.\ \cite{chiquet} using Theorem~\ref{thm:clustchar}. 
Consider a $\bar \lambda \geq \lambda$ and its corresponding sum-of-norms cluster model:
\begin{equation}
\min_{\x_1,\ldots,\x_n}\frac{1}{2}\sum_{i=1}^n\norm{\x_i-\a_i}^2
+\bar \lambda \sum_{i<j}\norm{\x_i-\x_j}.
\label{eq:sonclust3} 
\end{equation}

\begin{corollary}
(Chiquet et al.)
If there is a $C$ such that minimizer $\x^*$ of \eqref{eq:son-clustering} satisfies $\x_i^*=\hat\x$ for $i\in C$, $\x_i^*\ne\hat\x$ for $i\notin C$ for some $\hat\x \in \R^d$, then there exists an $\hat\x' \in \R^d$ such that the minimizer of \eqref{eq:sonclust3}, $\bar\x^*$, satisfies $\bar\x_i^*=\hat\x'$ for $i\in C$.
\end{corollary}

The corollary follows from Theorem~\ref{thm:clustchar}.  If $C$ is a cluster in the solution of \eqref{eq:son-clustering}, then by the necessary condition, there exist multipliers $\z_{ij}^*$ satisfying \eqref{eq:zstar1} for $\lambda$.  If we scale each of these multipliers by $\lambda/\bar\lambda$, we now obtain a solution to \eqref{eq:zstar1} for with $\lambda$ replaced by $\bar\lambda$, and the theorem states that this is sufficient for the points in $C$ to be in the same cluster in the solution to \eqref{eq:sonclust3}.

It should be noted that Hocking et al.\ construct an example of unequally-weighted sum-of-norms clustering in which the agglomeration property fails.  It is still mostly an open question to characterize for which norms and for which families of unequal weights the agglomeration property holds.  Refer to Chi and Steinerberger \cite{Chi} for some recent progress.

\section{Mixture of Gaussians}
\label{sec:mainrecovery}

In this section, we present our main result about recovery
of mixture of Gaussians.  As noted in the introduction, a theorem stating that every point is labeled correctly is not possible in the setting of $n\rightarrow\infty$, so we settle for a theorem stating that points within a constant number of standard deviations from the means are correctly labeled.

\begin{theorem}
\label{thm:mainrecovery}
Let the vertices $\a_1,\ldots,\a_n\in\R^d$ be generated from a mixture of $K$ Gaussian distributions with parameters
$\bmu_1,\ldots,\bmu_K$, $\sigma_1^2,\ldots,\sigma_K^2$, and $w_1,\ldots,w_K$. Let $\theta>0$ be given, and
let
\[V_m = \{\a_i: \norm{\a_i - \bmu_m} \leq \theta\sigma_m\},\quad m=1,\dots,K.\]
Let $\epsilon>0$ be arbitrary.
Then for any $m=1,\ldots,K$,
with probability exponentially close to $1$ 
(and depending on $\epsilon$) as $n\rightarrow \infty$, for the solution $\x^*$ computed by \eqref{eq:son-clustering},
the points in $V_m$ are in the same cluster 
provided
\begin{equation}
\lambda \ge \frac{2\theta\sigma_m}{(F(\theta,d)w_m-\epsilon)n}.
\label{eq:mainrecoverylambdalb}
\end{equation}
Here, $F(\theta,d)$ denotes the cumulative density function of the chi-squared distribution with $d$ degrees of freedom (which tends to $1$  rapidly as $\theta$ increases).  
Furthermore, the cluster associated with $V_m$ is distinct from the cluster associated with $V_{m'}$, $1\le m<m'<k$, provided that
\begin{equation}
\lambda <\frac{\norm{\bmu_{m}-\bmu_{m'}}}{2(n-1)}.
\label{eq:mainrecoverylambdaub}
\end{equation}
\label{thm:recovery}
\end{theorem}

\begin{proof}
 Let $\epsilon>0$ be fixed.  Fix an $m\in\{1,\ldots,K\}$. 
First, we show that all the points in $V_m$ are in the same
cluster.
The usual technique for proving a recovery result is to find subgradients to satisfy the sufficient condition, which in this case is Theorem~\ref{thm:clustchar} taking $C$ in the theorem to be $V_m$.  Observe that conditions \eqref{eq:zstar1} involve equalities and norm inequalities.  A standard technique in the literature (see, e.g., Cand\`es and Recht \cite{Candes2009}) is to find the least-squares solution to the equalities and then prove that it satisfies the inequalities.  This is the technique we adopt herein.  The conditions \eqref{eq:zstar1} are in sufficiently simple form that we can write down the least-squares solution in closed form; it turns out to be: 
\[\z_{ij}^* = \frac{1}{\lambda |V_m|}(\a_i - \a_j) \quad \forall i,j\in V_m, i\ne j.\]
It follows by construction (and is easy to check) that this formula satisfies the equalities in \eqref{eq:zstar1}, so the remaining task is to show that the norm bound $\norm{\z_{ij}^*}\le 1$ is satisfied.  By definition  of $V_m$, $\Vert\a_i-\a_j\Vert\le 2\theta\sigma_m$.   The probability that an arbitrary sample $\a_i$ is associated with mean $\bmu_m$ is $w_m$.  
Furthermore, with probability $F(\theta,d)$, this sample 
satisfies $\norm{\a_i-\bmu_m}\le \theta\sigma_m$, i.e., lands in $V_m$.  Since the second choice in the mixture of Gaussians is conditionally independent from the first, the overall probability that $\a_i$ lands in $V_m$ is $F(\theta,d)w_m$.  Therefore, $E[|V_m|]=F(\theta,d)w_mn$.  By the Chernoff bound for the tail of a binomial distribution,  it follows that the probability that $|V_m|\ge (F(\theta,d)w_m-\epsilon)n$ is exponentially close to 1 for a fixed $\epsilon>0$. Thus, provided
$\lambda \ge 2\theta\sigma_m/((F(\theta,d)w_m-\epsilon)n)$, we have constructed a solution to \eqref{eq:zstar1} with probability exponentially close to 1.

For the second part of the theorem, suppose $1\le m<m'\le K$.  For each sample $\a_i$ associated with $\bmu_m$ satisfying
$\Vert\a_i-\bmu_m\Vert\le \theta\sigma_m^2$ (i.e., lying in $V_m$),  the probability
is $1/2$ that \[(\a_i-\bmu_m)^T(\bmu_{m'}-\bmu_m)\le 0\] by the fact that the spherical Gaussian distribution has mirror-image symmetry about any hyperplane through its mean.  Therefore, with probability exponentially close to 1 as $n\rightarrow\infty$, we can assume that at least one $\a_i\in V_m$ satisfies the above inequality.
Similarly, with probability exponentially close to 1, at least one sample $\a_{i'}\in V_{m'}$ satisfies \[(\a_{i'}-\bmu_{m'})^T(\bmu_{m}-\bmu_{m'})\le 0.\]
Then 
\begin{align}
\norm{\a_i-\a_{i'}}^2 &=\norm{\a_i-\bmu_m-\a_{i'}+\bmu_{m'}+\bmu_m-\bmu_{m'}}^2 \notag \\
&=\norm{\a_i-\bmu_m-\a_{i}+\bmu_{m'}}^2 +
2(\a_i-\bmu_m)^T(\bmu_{m}-\bmu_{m'}) \notag \\
&\hphantom{=}\quad\mbox{}-2(\a_{i'}-\bmu_{m'})^T(\bmu_m-\bmu_{m'}) + \norm{\bmu_{m}-\bmu_{m'}}^2 \notag\\
&\ge \norm{\bmu_{m}-\bmu_{m'}}^2, \label{eq:aimumineq}
\end{align}
where, in the final line, we used the two inequalities derived earlier in this paragraph.

Consider the first-order optimality conditions for equation \eqref{eq:son-clustering}, which are given by \eqref{eq:KKTcond}. 
Apply the triangle inequality to the summation in \eqref{eq:KKTcond} to obtain,
\begin{align}
\norm{\x_i^* - \a_i} &\leq \lambda(n-1),\mbox{ and } \label{xisai} \\
\norm{\x_{i'}^* - \a_{i'}} &\leq \lambda(n-1).\label{xisai2}
\end{align}
Therefore,
\begin{align*}
    \norm{\x_i^*-\x_{i'}^*} &=\norm{\a_{i}-\a_{i'}+\x_i^*-\a_i-\x_{i'}^*+\a_{i'}} \\
    &\ge \norm{\a_{i}-\a_{i'}}-\norm{\x_i^*-\a_i}-\norm{\x_{i'}^*-\a_{i'}} && \mbox{(by the triangle inequality)} \\
    &\ge \norm{\bmu_{m'}-\bmu_m}-2\lambda(n-1) &&\mbox{(by \eqref{eq:aimumineq}, \eqref{xisai}, and \eqref{xisai2}).}
\end{align*}
Therefore, we conclude that $\x_i^*\ne \x_{i'}^*$, i.e., that $V_m$ and $V_{m'}$ are not in the same cluster, provided that the right-hand side of the preceding inequality is positive, i.e., 
\[\lambda <\frac{\norm{\bmu_{m}-\bmu_{m'}}}{2(n-1)}. \]
This concludes the proof of the second statement.
\end{proof}

Clearly, there exists a $\lambda$ so that the solution to \eqref{eq:son-clustering} can simultaneously place all points in $V_m$ into the same cluster for each $m=1,\ldots,K$ while distinguishing the clusters provided that the right-hand side of \eqref{eq:mainrecoverylambdaub} exceeds the right-hand side of \eqref{eq:mainrecoverylambdalb}.  
In order to obtain a compact inequality that guarantees this condition, let us fix some values.  For example, let us take $\theta=2d$ and let $c_d=F(2d,d)$.  The Chernoff bound implies that $c_d\rightarrow 1$ exponentially fast in $d$.   Let $w_{\min}$ be the minimum weight in the mixture of Gaussians. Let $\sigma_{\max}$ denote the maximum standard deviation in the distribution.  Finally, let us take $\epsilon=c_dw_{\min}/2$.  Then the above theorem states there is a $\lambda$ such that with probability tending to $1$ exponentially fast in $n$, the points in $V_m$, for any $m=1,\ldots,K$ are each in the same cluster, and these clusters are distinct, provided that
\begin{equation}
\min_{1\le m<m'\le K}\norm{\bmu_m-\bmu_{m'}}>\frac{16d\sigma_{\max}}{c_dw_{\min}}.
\label{eq:ourbound}
\end{equation}
Compared to the Panahi et al.\ bound \eqref{eq:panahibound}, we have removed the dependence of the right-hand side on $n$ as well as the factor of $K$.  (The dependence of the Panahi et al.\ bound on $d$ is not made explicit so we cannot compare the two bounds' dependence on $d$.  Note that there is still an implicit dependence on $K$  in \eqref{eq:ourbound} since necessarily $w_{\min}\le 1/K$.)

\section{Discussion}
The analysis of the mixture of Gaussians in the preceding section used only standard bounds and simple properties of the normal distribution, so it should be apparent to the reader that many extensions of this result (e.g., Gaussians with a more general covariance matrix, uniform distributions, many kinds of deterministic distributions) are possible.  The key technique is Theorem~\ref{thm:clustchar}, which essentially decouples the clusters from each other so that each can be analyzed in isolation.   Such a theorem does not apply to most other clustering algorithms, or even to sum-of-norm clustering in the case of unequal weights, so obtaining similar results for other algorithms remains a challenge.  
\bibliography{optimization}
\bibliographystyle{plain}
\end{document}